\documentclass[12pt]{alt2021arxiv} 

\usepackage{amsmath}
\usepackage{amsfonts}
\usepackage[makeroom]{cancel}

\usepackage{sty/notations}

\usepackage{enumerate}

\usepackage{pgf}
\usepackage{tikz}
\usetikzlibrary{arrows,automata}
\usetikzlibrary{positioning}

\usepackage[most]{tcolorbox}
\newtcolorbox{blockquote}{colback=orange!15!white,grow to right by=-1mm,grow to left by=-1mm,boxrule=0pt,boxsep=0pt,breakable}

\usepackage{todonotes}

\title[Minimax Lower Bounds for Episodic RL Revisited]{Episodic Reinforcement Learning in Finite MDPs: \\ Minimax Lower Bounds Revisited}
\usepackage{times}

\altauthor{%
 \Name{Omar Darwiche Domingues} \Email{omar.darwiche-domingues@inria.fr}\\
 \addr Inria Lille, SequeL team
 \AND
 \Name{Pierre M\'enard} \Email{pierre.menard@inria.fr}\\
 \addr Inria Lille, SequeL team
 \AND
 \Name{Emilie Kaufmann} \Email{emilie.kaufmann@univ-lille.fr}\\
 \addr CNRS \& ULille (CRIStAL), Inria Lille, SequeL team
  \AND
 \Name{Michal Valko} \Email{valkom@deepmind.com}\\
 \addr DeepMind Paris
}

\begin{document}

\maketitle

\begin{abstract}
In this paper, we propose new problem-independent lower bounds on the sample complexity and regret in episodic MDPs, with a particular focus on the \emph{non-stationary case} in which the transition kernel is allowed to change in each stage of the episode. Our main contribution is a novel lower bound of $\Omega((H^3SA/\epsilon^2)\log(1/\delta))$ on the sample complexity of an $(\varepsilon,\delta)$-PAC algorithm for best policy identification in a non-stationary MDP. This lower bound relies on a construction of ``hard MDPs'' which is different from the ones previously used in the literature. Using this same class of MDPs, we also provide a rigorous proof of the $\Omega(\sqrt{H^3SAT})$ regret bound for non-stationary MDPs. Finally, we discuss connections to PAC-MDP lower bounds.
\end{abstract}



	\section{Introduction}

	In episodic reinforcement learning (RL), an agent interacts with an environment in episodes of length $H$. In each stage $h \in \braces{1, \ldots, H}$, the agent is in a state $s_h$, takes an action $a_h$ then observes the next state $s_{h+1}$ sampled according to a transition kernel $p_h(\cdot|s_h, a_h)$, and receives a reward $r_h(s_h, a_h)$. The quality of a RL algorithm, which adaptively selects the next action to perform based on past observation, can be measured with different performance metrics.

	On the one hand, the sample complexity quantifies the number of episodes in which an algorithm makes mistakes (in the PAC-MDP setting) or the number of episodes needed before outputing a near optimal policy (in the best policy identification setting). On the other hand, the regret quantifies the difference between the total reward gathered by an optimal policy and that of the algorithm. Minimax upper bounds on the sample complexity or the regret of episodic RL algorithms in finite MDPs have been given in the prior work, for instance in the work of \citet{dann2015sample,Dann2017,Azar2017, Jin2018}, and \citet{Zanette2019}. Deriving \emph{lower bounds} is also helpful to assess the quality of these upper bounds, in particular in terms of their scaling in the horizon $H$, the number of states $S$ and the number of actions $A$.

    \paragraph{Sample complexity lower bounds} Sample complexity has mostly been studied in the $\gamma$-discounted setting for PAC-MDP algorithms \citep{Kakade03PhD}, for which the number of time steps in which an algorithm acts $\varepsilon$-sub-optimaly (called the \emph{sample complexity}) has to be upper bounded, with probability larger than $1-\delta$. State-of-the art lower bounds are a $\Omega\pa{\tfrac{SA}{\epsilon^2}\log\pa{\tfrac{S}{\delta}}}$ bound by \cite{Strehl09} and a $\Omega\left({\tfrac{SA}{(1-\gamma)^3\epsilon^2}\log\pa{\tfrac{1}{\delta}}}\right)$ bound by \cite{lattimore2012pac}.  A lower bound of the same order is provided by \cite{azar2012sample} for the number of steps algorithms that have access to a generative model need to identify an $\varepsilon$-optimal policy.

    PAC-MDP algorithms in the episodic setting were later studied by \cite{dann2015sample}, who also provide a lower bound. Unlike the previous ones, they \emph{do not }lower bound the number of $\varepsilon$-mistakes of the algorithm, but rather state that any algorithm that outputs a deterministic policy $\widehat{\pi}$ that is $\epsilon$-optimal with probability at least $1-\delta$, there exists an MDP where the expected number of episodes before $\widehat{\pi}$ is returned must be at least $\Omega\pa{\frac{SAH^2}{\epsilon^2}\log\pa{\frac{1}{\delta}}}$. This lower bound therefore applies to the sample complexity of best-policy identification (see Section~\ref{sec:setting} for a formal definition), which is our main focus in this paper. The ``hard MDP'' instances used to prove this worse-case bound are inspired by the the ones of \citet{Strehl09} and consist in $S$ multi-armed bandit problems played in parallel. While these hard instances all have transition kernels that are identical for each step $h$ of the episode (i.e. $p_h(\cdot|s_h,a_h)$ does not depend on $h$), we propose in this paper an alternative class of hard instances in which the transitions may be stage-dependent. We prove in Theorem~\ref{theorem:bpi-non-stationary} that there exists an MDP in this class for which the expected number of samples needed to identify an $\varepsilon$-optimal policy with probability $1-\delta$ is at least $\Omega\pa{\tfrac{SAH^3}{\epsilon^2}\log\pa{\tfrac{1}{\delta}}}$.

	\paragraph{Regret lower bounds} In the average-reward setting, \citet{jaksch2010near} prove a regret lower bound of $\Omega(\sqrt{DSAT})$ where $D$ is the diameter of the MDP and $T$ is the total number of actions taken in the environment. In the episodic setting, the total number of actions taken is $HT$, where $T$ is now the number of episodes, and $H$ is roughly the equivalent of the diameter $D$.\footnote{The diameter $D$ is the minimum average time to go from one state to another. In an episodic MDP, if the agent can come back to the same initial state $s_1$ after $H$ steps, the average time between any pair of states is bounded by $2H$, if we restrict the state set to the states that are reachable from $s_1$ in $H$ steps.}
	Hence, intuitively, the lower bound of \citet{jaksch2010near} should be translated to $\Omega(\sqrt{H^2SAT})$ for episodic MDPs after $T$ episodes. Yet, to the best of our knowledge, a precise proof of this claim has not been given in the literature. The proof of \citet{jaksch2010near} relies on building a set of hard MDPs with ``bad'' states (with zero reward) and ``good'' states (with reward 1), and can be adapted to episodic MDPs by making the good states absorbing. However, this construction does not include MDPs whose transitions are allowed to change at every stage $h$. In the case of stage-dependent transitions, \cite{Jin2018} claim that the lower bound becomes $\Omega(\sqrt{H^3SAT})$, by using the construction of \cite{jaksch2010near} and a mixing-time argument, but they do not provide a complete proof. In Theorem~\ref{theorem:regret-non-stationary}, we provide a detailed proof of their statement, by relying on the same class of hard MDPs given for our sample complexity lower bound.

	\paragraph{Our contributions} Our main contribution are unified, simple, and complete proofs of minimax lower bounds for different episodic RL settings. In particular, using a single class of hard MDPs and the same information-theoretic tools, we provide regret and sample complexity lower bounds for episodic reinforcement learning algorithms for stage-dependent transitions. For~$T$ episodes, the regret bound is $\Omega(\sqrt{H^3SAT})$, which is the same as the one sketched by \cite{Jin2018}, but we provide a detailed proof with a different construction. This lower bound is matched by the optimistic Q-learning algorithm of \citet{Jin2018}. For the sample complexity of best-policy identification (BPI), we prove the first lower bound for MDPs with stage-dependent transitions and for algorithms that may output randomized policies. This bound is of order $\Omega\pa{\frac{SAH^3}{\epsilon^2}\log\pa{\frac{1}{\delta}}}$ and is matched by the BPI-UCBVI algorithm of \cite{RFExpress}. As a corollary of the BPI lower bound, we also obtain a lower bound of $\Omega\pa{\frac{SAH^3}{\epsilon^2}\log\pa{\frac{1}{\delta}}}$ in the PAC-MDP setting.
	Finally, note that our proof technique also provides rigorous proofs of the bounds $\Omega(\sqrt{H^2SAT})$ and $\Omega\pa{\frac{SAH^2}{\epsilon^2}\log\pa{\frac{1}{\delta}}}$ for regret and best-policy identification with stage-independent transitions.

	\section{Setting and Performance Measures}\label{sec:setting}

	\paragraph{Markov decision process} We consider an episodic Markov decision process (MDP) defined as a tuple $\mdp\eqdef(\cS, \cA, H, \mu, p, r)$ where $\cS$ is the set of states, $\cA$ is the set of actions, $H$ is the number of steps in one episode, $\mu$ is the initial state distribution, $p = \braces{p_h}_h$ and $r = \braces{r_h}_h$ are sets of transitions and rewards for $h \in [H]$ such that taking an action $a$ in state $s$ results in a reward $r_h(s, a) \in [0, 1]$ and a transition to $s' \sim p_h(\cdot|s,a)$. We assume that the cardinalities of $\cS$ and $\cA$ are finite, and  denote them by $S$ and $A$, respectively.
	
	\paragraph{Markov and history-dependent policies }
	Let $\simplex(\cA)$ be the set of probability distributions over the action set and let 
	\begin{align*}
		\histset_h^{\,t} =   \pa{(\cS\times\cA)^{H-1}\times \cS}^{t-1} \times (\cS\times\cA)^{h-1}\times \cS
	\end{align*}
	be the set of possible histories up to step $h$ of episode $t$, that is, the set of tuples of the form 
	\begin{align*}
		\left(s_1^1, a_1^1, s_2^1, a_2^1, \ldots, s_H^1, \ldots, s_1^t, a_1^t, s_2^t, a_2^t, \ldots, s_h^t\right) \in \histset_h^{\,t}.
	\end{align*}
		A \textit{Markov policy} is a function $\pi: \cS \times [H] \to \simplex(\cA)$ such that $\pi(a|s, h)$ denotes the probability of taking action $a$ in state $s$ at step $h$. 
	A \textit{history-dependent policy} is a family of functions denoted by  $\histpi \eqdef (\pi_h^t)_{t\geq 1, h\in[H]},$ where $\pi_h^t: \histset_h^{\,t} \to \simplex(\cA)$ such that $\pi_h^t(a\,|\,i_h^t)$ denotes the probability of taking action $a$ at time $(t, h)$ after observing the history  $i_h^t \in \histset_h^{\,t}$.
	We denote by $\markovpolicyset$ and $\histpolicyset$ the sets of Markov and history-dependent policies, respectively.

	\paragraph{Probabilistic model} A policy $\histpi$ interacting with an MDP defines a stochastic process denote by $(S_h^t, A_h^t)_{t\geq 1, h\in [H]}$, where $S_h^t$ and $A_h^t$ are the random variables representing the state and the action at time $(t, h)$. As explained by \citep{lattimore2020bandit}, the Ionescu-Tulcea theorem ensures the existence of probability space $(\Omega, \cF, \P_{\mdp})$ such that 
	\begin{align*}
		\PPs{\mdp}{S_1^t = s} = \mu(s),\
		\PPs{\mdp}{S_{h+1}^t = s | A_h^t, \hist_h^t} = p_h(s | S_h^t, A_h^t)
		,\text{ and}\quad 
		\PPs{\mdp}{A_h^t = a | \hist_h^t} = \pi_h^t( a | \hist_h^t ),
	\end{align*}
	where  $\histpi = (\pi_h^t)_{t\geq 1, h\in[H]}$ and for any $(t,h)$,
	\begin{align*}
		\hist_h^t \eqdef \left(S_1^1, A_1^1, S_2^1, A_2^1, \ldots, S_H^1, \ldots, S_1^t, A_1^t, S_2^t, A_2^t, \ldots, S_h^t\right)
	\end{align*}
	is the random vector taking values in $\histset_h^{\,t}$ containing all state-action pairs observed up to step $h$ of episode $t$, but not including $A_h^t$. We denote by $\cF_h^t$ the $\sigma$-algebra generated by $\hist_h^t $. Next, we denote by $\P_{\mdp}^{\hist_H^T}$ the pushforward measure of $\hist_H^T$ under $\P_{\mdp},$
	\begin{align*}
	 \PPsu{\mdp}{\hist_H^T}{i_H^T} 
	  & \eqdef 
	   \PPs{\mdp}{\hist_H^T = i_H^T}
	   =  \prod_{t=1}^T \mu(s_1^t) \prod_{h=1}^{H-1} \pi_h^t(a_h^t|i_h^t) p_h\left(s_{h+1}^t|s_h^t, a_h^t\right),
	\end{align*}
	where $i_h^t \eqdef (s_1^1, a_1^1, \ldots, s_H^1, \ldots, s_1^t, a_1^t, \ldots, s_h^t) \in \histset_h^{\,t}$. Moreover, let $\E_{\mdp}$ be the expectation under~$\P_{\mdp}$. Notice that the dependence of $\P_{\mdp}$ and $\E_{\mdp}$ on the policy $\histpi$ is denoted implicitly, and we denote them explicitly as $\P_{\histpi, \mdp}$ and $\E_{\histpi,\mdp}$ when it is relevant.

	\paragraph{Value function} In an episode $t$, the value of a policy $\histpi$ in the MDP $\mdp$ is defined as
	\begin{align*}
		V^{\histpi, t}(i_H^{t-1}, s) \eqdef \EEs{\histpi, \mdp}{\sum_{h=1}^H r_h(S_h^t, A_h^t) \Big| I_H^{t-1} = i_H^{t-1}, S_1^{\,t}=s},
	\end{align*}
	where $i_H^{t-1}$ are the states and actions observed before episode $t$ and $\histpi$ can be history-dependent. In particular, for a Markov policy~$\pi$, the value does not depend on $i_H^{t-1}$ and we have
	\begin{align*}
		V^{\pi}(s) \eqdef \EEs{\pi, \mdp}{\sum_{h=1}^H r_h(S_h^1, A_h^1) \Big| S_1^1 = s}.
	\end{align*}
	The optimal value function $V^*$ is defined as
		$V^*(s) \eqdef \max_{\pi\in\markovpolicyset} V^{\pi}(s)$ which is achieved by an optimal policy $\pi^*$ that satisfies $V^{\pi^*}(s) = V^*(s)$ for all $s\in\cS$.
	As a consequence of Theorem 5.5.1 of \cite{puterman2014markov}, we have $V^*(s) \geq V^{\histpi, t}(i_H^{t-1}, s)$, which shows that Markov policies are sufficient to achieve an optimal value function.
We also define $\rho^* \eqdef \rho^{\pi^*}$ and the average value functions over the initial state as
	\begin{align*}
	\rho^{\histpi, t}(i_H^{t-1}) \eqdef \EEs{s\sim \mu}{V^{\histpi, t}(i_H^{t-1}, s)}
	,\quad 
	\rho^\pi \eqdef \EEs{s\sim \mu}{V^\pi(s)}.
	\end{align*}
	
	\paragraph{Algorithm} We define a reinforcement-learning algorithm as a history-dependent policy $\histpi$ used to interact with the environment. In the BPI setting, where we eventually stop and recommend a policy, an algorithm is defined as a triple $(\histpi, \tau, \recpolicy_\tau)$ where $\tau$ is a stopping time with respect to the filtration $(\cF_H^t)_{t\geq 1}$, and $\recpolicy_\tau$ is a Markov policy recommended after $\tau$ episodes.

	
	\paragraph{Performance criteria} The performance of RL algorithms has been commonly measured according to its regret or under a Probably Approximately
	Correct (PAC) framework, as defined below.
	\begin{blockquote}
	\begin{definition}\label{def:regret} The expected regret of an algorithm $\histpi$ in an MDP $\mdp$  after $T$ episodes is defined as $$\regret_T(\histpi, \mdp) \eqdef  \EEs{\histpi,\mdp}{\sum_{t=1}^T \pa{\rho^* - \rho^{\histpi, t}(I_H^{t-1})}}.$$
	\end{definition}
	\end{blockquote}
	
	\begin{blockquote}
	\begin{definition}\label{def:pac} An algorithm $\histpi$ is $(\epsilon, \delta)$-PAC for exploration in an MDP $\mdp$ (or PAC-MDP) if there exists a polynomial function $F_\mathrm{PAC}(S, A, H, 1/\epsilon, \log(1/\delta))$ such that its \emph{sample complexity} 
		$$\scPAC_\epsilon \eqdef \sum_{t=1}^\infty \indic{ \rho^* - \rho^{\histpi, t}(I_H^{t-1}) > \epsilon}$$
	satisfies
		 $\PPs{\histpi,\mdp}{\scPAC_\epsilon > F_\mathrm{PAC}(S, A, H, 1/\epsilon, \log(1/\delta))} \leq \delta.$  
	\end{definition}
	\end{blockquote}

	\begin{blockquote}
	\begin{definition} \label{def:bpi} An algorithm $(\histpi, \tau, \recpolicy_\tau)$ is $(\epsilon, \delta)$-PAC for best-policy identification in an MDP~$\mdp$ if the policy $\recpolicy_\tau$ returned after $\tau$ episodes satisfies 
		$$\PPs{\histpi,\mdp}{\rho^{\recpolicy_\tau} \leq \rho^* - \epsilon} \leq \delta.$$
		The \emph{sample complexity} is defined as the number of episodes $\tau$ required for stopping. 
	\end{definition}
	\end{blockquote}

\section{Lower Bound Recipe}
\label{sec:receipe}

In this section, we present the two main ingredients for the proof of our minimax lower bounds. These lower bounds consider a class $\cC$ of hard MDPs instances (on which the optimal policy is difficult to identify), that are typically close to each other, but for which the behavior of an algorithm is expected to be different (because they do not share the same optimal policy). The class $\cC$ used to derive all our results is presented in Section~\ref{sec:hard-mdps}. Then, lower bound proofs use a change of distribution between two well-chosen MDPs in~$\cC$ in order to obtain inequalities on the expected number of visits of certain state-action pairs in one of them. The information-theoretic tools that we use for these changes of distributions are gathered in Section~\ref{sec:CD}.

\subsection{Hard MDP instances} \label{sec:hard-mdps}

From a high-level perspective, the family of MDPs that we use for our proofs behave like multi-armed bandits with $\Theta(HSA)$ arms. To gain some intuition about the construction, assume that $S=4$ and consider the MDP in Figure~\ref{fig:mdp-non-stationary-fixed-s}. The agent starts in a \emph{waiting state} $\swait$ where it can take an action $\await$ to stay in $\swait$ up to a stage $\barH < H$, after which the agent has to leave $\swait$. From $\swait$, the agent can only transition to a state $s_1$, from which it can reach two absorbing states, a ``good'' state $s_g$ and a ``bad'' state $s_b$. The state $s_g$ is the only state where the agent can obtain a reward, which starts to be 1 at stage $\barH+2$. There is a single action $a^*$ in state $s_1$ that increases by $\epsilon$ the probability of arriving to the good state, and this action must be taken at a specific stage $h^*$. The intuition is that, in order to maximize the rewards, the agent must choose the right moment $h\in \braces{1,\ldots, \barH}$ to leave~$\swait$, and then choose the good action $a^*\in\braces{1, \ldots, A}$ in $s_1$. This results in a total of $\barH A$ possible choices, or ``arms'' and the maximal reward is $\Theta(\barH)$. By analogy with the existing minimax regret bound for multi-armed bandits \citep{auer2002nonstochastic,lattimore2020bandit}, the regret lower bound should be $\Omega(H\sqrt{H A T})$, by taking $\barH=\Theta(H)$.

 \begin{figure}[ht!]
	\centering
	\begin{tikzpicture}[->,>=stealth',shorten >=1pt,auto,node distance=2.8cm,
semithick]
\tikzstyle{every state}=[fill=white,text=black]

\node[state]                (W)                  {$s_{\mathrm{w}}$};
\node[state]                (A) [below=1cm of W]    {$s_1$};
\node[state, fill=white]    (B) [below right of=A] {$s_g$};
\node[state, fill=white]    (C) [below left  of=A] {$s_b$};

\node[state, fill=white,draw=none]    (RG) [right=0.25cm of B] {$r_h(s_g, a)= \indic{h \geq \overline{H}+2}$};
\node[state, fill=white,draw=none]    (RB) [left=0.25cm of C] {$r_h(s_b, a)=0$};

\path  
	(W)  edge node{action  $\neq a_{\mathrm{w}}$} (A)
	(W)  edge [loop above] node{action = $a_{\mathrm{w}}$} (W)  
	(A)  edge[blue]   node{$\textcolor{blue}{\frac{1}{2}+\epsilon'} $} (B)
	(A)  edge[bend right,dashed]   node[below]{$\frac{1}{2}$} (B)
	(A)  edge[bend left,dashed]   node[below]{$\frac{1}{2}$} (C)
	(A)  edge[blue]   node[above left]{$\textcolor{blue}{\frac{1}{2}-\epsilon'}$} (C)
	(B)  edge [loop below] node{1} (B)  
	(C)  edge [loop below] node{1} (C)    
	;

\end{tikzpicture}

%
%
	\caption{Illustration of the class of hard MDPs for $S=4$.}
	\label{fig:mdp-non-stationary-fixed-s}
\end{figure}
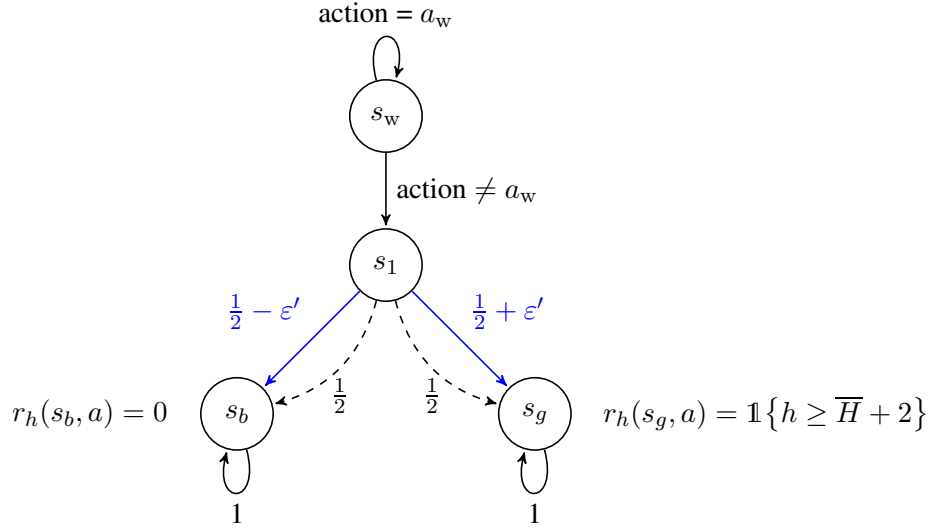
\noindent
Inspired by the tree construction of \cite{lattimore2020bandit} for the lower bound in the average-reward setting, we now generalize these MDPs to $S>4$.
 Consider a family of MDPs described as follows and illustrated in Figure \ref{fig:mdp-non-stationary}. First, we state the following assumption, which we relax in Appendix~\ref{sec:relax-assumption}.

\begin{assumption}
	\label{assumption:mdp-structure}
	The number of states and actions satisfy $S \geq 6$, $A \geq 2$, and there exists an integer~$d$ such that $S = 3 + (A^d-1)/(A-1)$, which implies $d= \Theta(\log_A S)$. We further assume that $H \geq 3 d$.
\end{assumption}
 As in the previous case, there are three special states: a ``waiting'' state $\swait$ where the agent starts and can choose to stay up to a stage $\barH$, a ``good'' state $s_g$ that is absorbing and is the only state where the agent obtains rewards, and a ``bad'' state $s_b$ that is absorbing and gives no reward. The other $S-3$ states are arranged in a full $A$-ary tree of depth $d-1$, which can be done since we assume there exists an integer $d$ such that $S-3 = \sum_{i=0}^{d-1}A^i$. The root of the tree is denoted by $\sroot$, which can only be reached from $\swait$, and the states $s_g$ and $s_b$ can only be reached from the leaves of the tree.

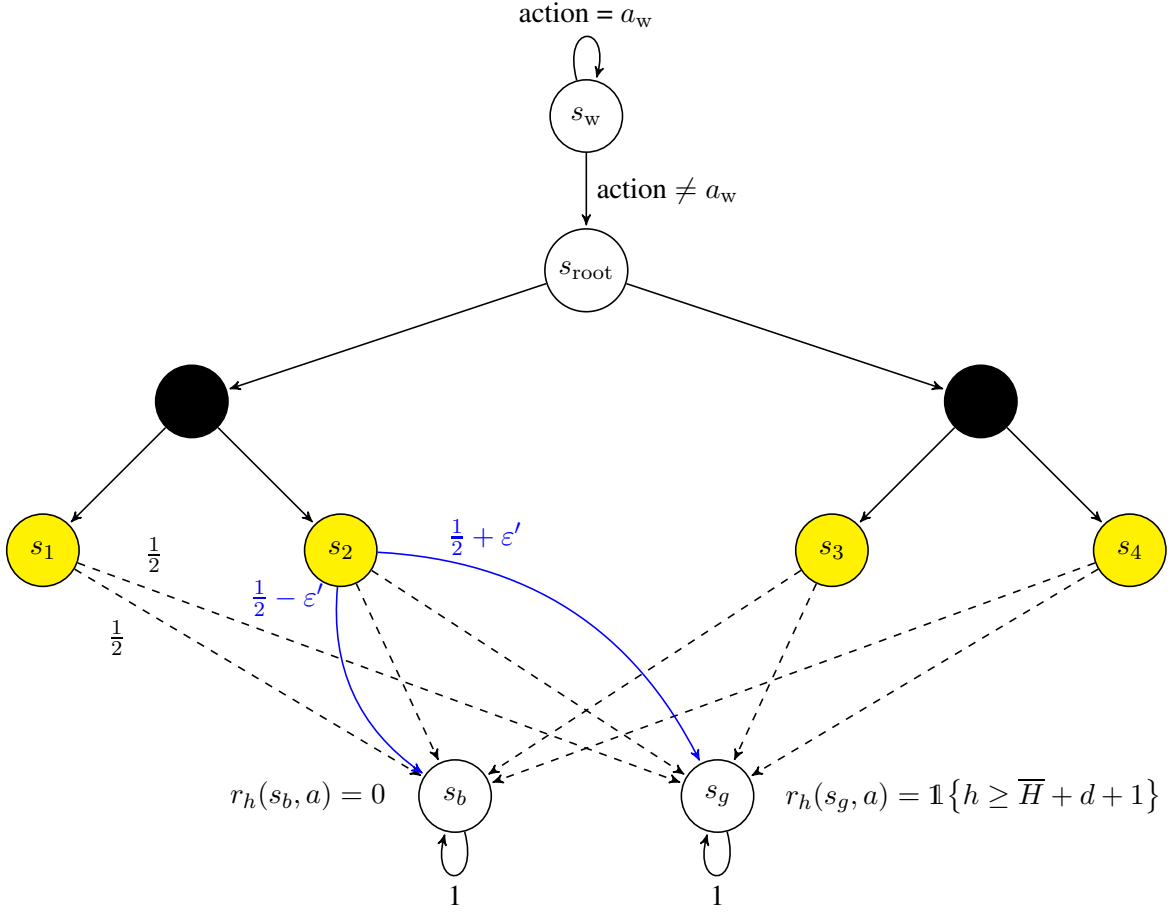
\begin{figure}[ht]
	\centering

\begin{tikzpicture}[->,>=stealth',shorten >=1pt,auto,node distance=2.8cm,
semithick]
\tikzstyle{every state}=[fill=white,text=black]

\node[state]                (W)                      {$s_{\mathrm{w}}$};
\node[state]                (A) [below=1cm of W]     {$s_\mathrm{root}$};


\node[state,fill=black]                (N1)[below left=1cm and 4.5cm of A]    {}; 
\node[state,fill=black]                (N2)[below right=1cm  and 4.5cm of A]    {}; 

\path 
	(A) edge (N1)
	(A) edge (N2);

\node[state,fill=yellow]     (N3)[below  left of=N1]    {$s_1$}; 
\node[state,fill=yellow]     (N4)[below  right of=N1]    {$s_2$}; 

\path 
	(N1) edge (N3)
	(N1) edge (N4);

\node[state,fill=yellow]     (N5)[below  left  of=N2]    {$s_3$}; 
\node[state,fill=yellow]     (N6)[below  right of=N2]    {$s_4$}; 

\path 
	(N2) edge (N5)
	(N2) edge (N6);

\path  
	(W)  edge node{action  $\neq a_{\mathrm{w}}$} (A)
	(W)  edge [loop above] node{action = $a_{\mathrm{w}}$} (W)  
	;
	
\node[state, fill=white]    (GOOD) [below right = 6.25cm and 1cm of A] {$s_g$};
\node[state, fill=white]    (BAD) [below left  = 6.25cm and 1cm of A] {$s_b$};

\node[state, fill=white,draw=none]    (RG) [right=0.25cm of GOOD] {$r_h(s_g, a) =\indic{h \geq \overline{H}+d+1}$};
\node[state, fill=white,draw=none]    (RB) [left=0.25cm of BAD] {$r_h(s_b, a)=0$};

\path
	(N3)  edge[dashed]   node[above left=1.2cm and 2.75cm]{$\frac{1}{2}$} (GOOD)
	(N3)  edge[dashed]   node[above left=0.1cm and 1.5cm]{$\frac{1}{2}$} (BAD)
	
	(N4)  edge[dashed]   node[below]{} (GOOD)
	(N4)  edge[dashed]   node[below]{} (BAD)
	
	(N5)  edge[dashed]   node[below]{} (GOOD)
	(N5)  edge[dashed]   node[below]{} (BAD)
	
	(N6)  edge[dashed]   node[below]{} (GOOD)
	(N6)  edge[dashed]   node[below]{} (BAD)
	;

\path 
	(N4)  edge[blue, bend left]   node[above left=0.6cm and 0.5cm]{$\textcolor{blue}{\frac{1}{2}+\epsilon'} $} (GOOD)
	(N4)  edge[blue, bend right]   node[above left=0.9cm and 0.25cm]{$\textcolor{blue}{\frac{1}{2}-\epsilon'} $} (BAD)
	(GOOD)  edge [loop below] node{1} (GOOD)  
	(BAD)  edge [loop below] node{1}  (BAD)    
	;
\end{tikzpicture}
	\caption{Illustration of the class of hard MDPs used in the proofs of Theorems \ref{theorem:bpi-non-stationary} and \ref{theorem:regret-non-stationary}.}
	\label{fig:mdp-non-stationary}
\end{figure}\noindent
Let $\barH \leq H-d$ be an integer that will be a parameter of the class of MDPs. Letting $\cL = \braces{s_1, s_2, \ldots, s_L}$ be the set of $L$ leaves of the tree, we define for each
\[(h^*, \ell^*, a^*) \in \braces{1+d,\ldots, \barH+d}\times \cL \times\cA,\]
an MDP $\mdp_{(h^*,\ell^*,a^*)}$ as follows.
For any state in the tree, the transitions are deterministic: the $a$-th action in a node leads to the $a$-th child of that node.
The transitions from $\swait$ are given by
\begin{align*}
& p_h(\swait|\swait, a) \eqdef \indic{a =  \await, h \leq \barH}
\quad\text{and}\quad
p_h(\sroot|\swait, a) \eqdef 1 -  p_h(\swait|\swait, a).
\end{align*}
That is, there is an action $\await$ that allows the agent to stay at $\swait$ up to a stage $\barH$. After stage $\barH$, the agent has to traverse the tree down to the leaves. The transitions from any leaf $s_i \in \cL$ are given by
\begin{align}
\label{eq:transition-probs}
& p_h(s_g|s_i, a) \eqdef \frac{1}{2} +  \Delta_{(h^*, \ell^*, a^*)}(h, s_i, a)
\quad\text{and}\quad
p_h(s_b|s_i, a) \eqdef \frac{1}{2} -  \Delta_{(h^*, \ell^*, a^*)}(h, s_i, a),
\end{align}
where $\Delta_{(h^*, \ell^*, a^*)}(h, s_i, a) \eqdef \indic{(h, s_i, a)=(h^*, s_{\ell^*}, a^*)}\cdot \epsilon'$, for some $\epsilon' \in [0,1/2]$ that is the second parameter of the class. This means that there is a single leaf $\ell^*$ where the agent can choose an action $a^*$ at stage $h^*$ that increases the probability of arriving to the good state $s_g$. Finally, the states $s_g$ and $s_b$ are absorbing, that is, for any action $a$, we have $p_h(s_b|s_b, a) \eqdef p_h(s_g|s_g, a) \eqdef 1$. The reward function depends only on the state and is defined as
\begin{align*}
 \forall a\in\cA, \quad r_h(s,a) \eqdef \indic{s = s_g, h \geq \barH+d+1}
\end{align*}
so that the agent does not miss any reward if it chooses to stay at $\swait$ until stage $\barH$.

We further define a reference MDP $\mdp_0$ which is an MDP of the above type but for which $\Delta_{0}(h, s_i, a)\eqdef0$ for all $(h,s_i,a)$. For every $\varepsilon'$ and $\barH$, we define the class $\cC_{\barH,\varepsilon'}$ to be the set
\[\cC_{\barH,\varepsilon'} \eqdef \left\{\mdp_0\right\} \bigcup \left\{\mdp_{(h^*,\ell^*,a^*)}\right\}_{(h^*,\ell^*,a^*) \in \braces{1+d,\ldots, \barH+d}\times \cL \times\cA}\;.\]

	\subsection{Change of Distribution Tools}\label{sec:CD}

	\begin{definition}
		The Kullback-Leibler divergence between two distributions $\P_1$ and $\P_2$ on a measurable space $(\Omega, \cG)$ is defined as
		\begin{align*}
			\KL(\P_1, \P_2) \eqdef \int_\Omega \log\pa{\frac{\rmd\P_1}{\rmd\P_2}(\omega)}\rmd\P_1(\omega),
		\end{align*}
		if $\P_1 \ll \P_2$ and $+\infty$ otherwise. For Bernoulli distributions, we define $	\forall (p,q)\in[0,1]^2,$
		\begin{align*}
			\kl(p, q) \eqdef \KL\big(\cB(p),\cB(q)\big) =p \log\pa{\frac{p}{q}} + (1-p)\log\pa{\frac{1-p}{1-q}}\cdot
		\end{align*}
	\end{definition}

%

	\begin{lemma}[proof in Appendix~\ref{sec:kl-between-mdps-with-stopping-time}]
		\label{lemma:kl-between-mdps-with-stopping-time}
		Let $\mdp$ and $\mdp'$ be two MDPs that are identical except for their transition probabilities, denoted by $p_h$ and $p_h'$, respectively. Assume that  we have $\forall (s, a)$, $p_h(\cdot|s,a) \ll p_h'(\cdot|s,a)$. Then, for any stopping time $\tau$ with respect to $(\cF_H^t)_{t\geq 1}$ that satisfies $\PPs{\mdp}{\tau < \infty} =1$,
		\begin{align}
		\label{eq:kl-between-mdps-stopping-time-aux}
		\KL\pa{\P_{\mdp}^{\hist_H^\tau}, \P_{\mdp'}^{\hist_H^\tau}} = \sum_{s \in \cS}\sum_{a \in \cA}\sum_{h\in[H-1]}
		\EEsu{\mdp}{}{ N_{h,s,a}^\tau } \KL\pa{ p_h(\cdot|s,a), p_h'(\cdot|s,a)},
		\end{align}
		where $N_{h,s,a}^\tau \eqdef \sum_{t=1}^\tau \indic{(S_h^t, A_h^t) = (s,a)}$ and $\hist_H^\tau: \Omega \to \bigcup_{t\geq 1} \histset_H^t: \omega \mapsto I_H^{\tau(\omega)}(\omega) $ is the random vector representing the history up to episode $\tau$.
	\end{lemma}

	\begin{lemma}[Lemma 1, \citealp{garivier2019explore}]
		\label{lemma:kl-contraction}
		Consider a measurable space $(\Omega, \cF)$ equipped with two distributions $\P_1$ and $\P_2$. For any $\cF$-measurable function $Z: \Omega\to[0, 1]$, we have
		\begin{align*}
			\KL(\P_1, \P_2) \geq \kl(\E_1[Z], \E_2[Z]),
		\end{align*}
		where $\E_1$ and $\E_2$ are the expectations under $\P_1$ and $\P_2$ respectively.
	\end{lemma}

\section{Sample Complexity Lower Bounds}

We are ready to state a new minimax lower bound on the sample complexity of best policy identification (see Definition~\ref{def:bpi}), in an MDP with stage-dependent transitions. We note that unlike existing sample complexity lower bounds which also construct ``bandit-like'' hard instances \citep{Strehl09,lattimore2012pac,dann2015sample}, we do not refer to the bandit lower bound of \cite{mannor2004sample}, but instead use explicit change of distribution arguments based on the tools given in Section~\ref{sec:CD}. This allows us to provide BPI lower bounds for algorithms that output randomized policies and to have a self-contained proof.
As a consequence of this result, we then easily derive a PAC-MDP (see Definition~\ref{def:pac}) lower bound in Corollary~\ref{corollary:pac}, which is proved in Appendix~\ref{sec:pac-proof}.

\begin{blockquote}
\begin{theorem}
	\label{theorem:bpi-non-stationary}
	 Let $(\histpi, \tau, \recpolicy_\tau)$ be an algorithm that is $(\epsilon, \delta)$-PAC for best policy identification in any finite episodic MDP. Then, under Assumption \ref{assumption:mdp-structure}, there exists an MDP $\mdp$ with stage-dependent transitions such that  for $\epsilon \leq H/24$, $H \geq 4$ and $\delta \leq 1/16$,
	\[\EEs{\histpi, \mdp}{\tau} \geq \frac{1}{3456}\frac{H^3SA}{\epsilon^2}\log\pa{\frac{1}{\delta}}\cdot\]
\end{theorem}
\end{blockquote}

\begin{blockquote}
	\begin{corollary}
		\label{corollary:pac}
		Let $\histpi$ be an algorithm that is $(\epsilon,\delta)$-PAC for exploration according to Definition~\ref{def:pac} and that, in each episode $t$, plays a deterministic policy $\pi_t$. Then, under the assumptions of Theorem~\ref{theorem:bpi-non-stationary}, there exists an MDP $\mdp$ such that 
		\begin{align*}
		\PPs{\histpi,\mdp}{\scPAC_\epsilon > \frac{1}{6912}\frac{H^3SA}{\epsilon^2}\log\pa{\frac{1}{\delta}}-1} > \delta.
		\end{align*}
	\end{corollary}
\end{blockquote}

\begin{proof}\textbf{of Theorem \ref{theorem:bpi-non-stationary}}  Without loss of generality, we assume that for any $\mdp$, the algorithm satisfies $\PPs{\histpi,\mdp}{\tau < \infty} = 1$. Otherwise,  there exists an MDP with  $\EEs{\histpi,\mdp}{\tau}=+\infty$ and the lower bound is trivial. 

 We will prove that the lower bound holds for the reference MDP $\cM_0$ defined in Section~\ref{sec:hard-mdps}, that has no optimal action. To do so, we will consider changes of distributions with other MDPs in the class $\cC_{\barH,\tepsilon}$ for $\barH$ to be chosen later and $\tepsilon \eqdef 2\epsilon/(H-\overline{H}-d)$. These MDPs are of the form $\mdp_{(h^*,\ell^*, a^*)}$ with $(h^*, \ell^*, a^*) \in \braces{1+d,\ldots, \barH+d}\times \cL \times\cA$, for which 
\begin{align*}
		\Delta_{(h^*, \ell^*, a^*)}(h,s_i,a) =  \indic{h=h^*,s_i=s_{\ell^*}, a=a^*} \,  \tepsilon,
	\end{align*}
We recall that $d-1$ is the depth of the tree.
We denote by $\P_{(h^*,\ell^*, a^*)} \eqdef\P_{\histpi, \mdp_{(h^*,\ell^*, a^*)}}$ and $\E_{(h^*,\ell^*, a^*)} \eqdef\E_{\histpi, \mdp_{(h^*,\ell^*, a^*)}}$ the probability measure and expectation in the MDP $\mdp_{(h^*,\ell^*, a^*)}$ by following $\histpi$ and by $\P_0$ and $\E_0$ the corresponding operators in the MDP $\mdp_0$.

	\paragraph{Suboptimality gap of $\bm{\recpolicy_\tau}$} We can show that the value of the optimal policy in any of the MDPs $\mdp_{(h^*, \ell^* a^*)}$ is
	$\rho^* = (H-\barH-d)\pa{\frac{1}{2}+\tepsilon}$
	and the value of the recommended policy $\recpolicy_\tau$ is
	\begin{align*}
		\rho_{(h^*,\ell^*,a^*)}^{\recpolicy_\tau} = (H-\barH-d)\pa{
			\frac{1}{2}
			+\, \tepsilon \, \bfPPsu{(h^*,\ell^*,a^*)}{\recpolicy_\tau}{S_{h^*} =s_{\ell^*}, A_{h^*}=a^*}
		}
	\end{align*}
	where $\bfPP_{(h^*,\ell^*,a^*)}^{\recpolicy_\tau}$ is the probability distribution over states and actions $(S_h, A_h)_{h\in[H]}$ following the Markov policy $\recpolicy_\tau$ in the MDP $\mdp_{(h^*, \ell^* a^*)}$. Notice that $\rho_{(h^*,\ell^*,a^*)}^{\recpolicy_\tau}$ is a random variable and $\bfPP_{(h^*,\ell^*,a^*)}^{\recpolicy_\tau}$ is a random measure that are $\cF_H^{\,\tau}$-measurable. Hence,
	\begin{align*}
		\rho^* - \rho_{(h^*,\ell^*,a^*)}^{\recpolicy_\tau} =
		 2\epsilon\left(1-\bfPPsu{(h^*,\ell^*,a^*)}{\recpolicy_\tau}{S_{h^*} =s_{\ell^*}, A_{h^*}=a^*}\right)
	\end{align*}
	and
	\begin{align*}
		\rho^* - \rho_{(h^*,\ell^*,a^*)}^{\recpolicy_\tau} < \epsilon
		\iff
		\bfPPsu{(h^*,\ell^*,a^*)}{\recpolicy_\tau}{S_{h^*} =s_{\ell^*}, A_{h^*}=a^*} > \frac{1}{2}\,.
	\end{align*}

	\paragraph{Definition of a ``good'' event $\bm{\cE_{(h^*,\ell^*,a^*)}^\tau}$ for $\bm{\mdp_{(h^*, \ell^* a^*)}}$}
	 The transitions of all MDPs are the same up to the stopping time $\eta = \min\braces{h\in[H]: S_h \in \cL}$ when a leaf is reached. Hence, $\eta$ depends only on the policy that is followed, and not on the parameters of the MDP, which allows us to define the random measure $\bfPP^{\recpolicy_\tau}$ as
	\begin{align}
	\bfPPu{\recpolicy_\tau}{S_{h^*} = s_{\ell^*}, A_{h^*}=a^*} & \eqdef
	 \bfPPsu{(h^*,\ell^*,a^*)}{\recpolicy_\tau}{S_{\eta} = s_{\ell^*}, A_{\eta}=a^*, \eta=h^*} \label{eq:aux-bpi-2}\\
	  & = \bfPPsu{(h^*,\ell^*,a^*)}{\recpolicy_\tau}{S_{h^*} =s_{\ell^*}, A_{h^*}=a^*}  \nonumber
	\end{align}
	since the probability distribution of $(S_\eta, A_\eta, \eta)$ on the RHS of \eqref{eq:aux-bpi-2} does not depend on the parameters of the MDP $(h^*,\ell^*,a^*)$, given $\eta =h^*$.
	We define the event
	\begin{align*}
		\cE_{(h^*,\ell^*,a^*)}^\tau \eqdef \braces{\bfPPu{\recpolicy_\tau}{S_{h^*} = s_{\ell^*}, A_{h^*}=a^*} > \frac{1}{2}},
	\end{align*}
	which is said to be ``good'' due to the fact that $\cE_{(h^*,\ell^*,a^*)}^\tau = \braces{\rho_{(h^*,\ell^*,a^*)}^{\recpolicy_\tau} > \rho^*-\epsilon}$. Since the algorithm is assumed to be $(\epsilon, \delta)$-PAC for any MDP, we have
	\begin{align*}
		\PPs{(h^*,\ell^*,a^*)}{\cE_{(h^*,\ell^*,a^*)}^\tau} = \PPs{(h^*,\ell^*,a^*)}{\rho_{(h^*,\ell^*,a^*)}^{\recpolicy_\tau} > \rho^*-\epsilon} \geq 1-\delta.
	\end{align*}

	\paragraph{Lower bound on the expectation of $\bm{\tau}$ in the reference MDP $\bm{\mdp_0}$} Recall that
	\begin{align*}
		N_{(h^*, \ell^*, a^*)}^\tau = \sum_{t=1}^\tau \indic{S_{h^*}^t=s_{\ell^*}, A_{h^*}^t=a^*}\;,
	\end{align*}
	such that $\sum_{(h^*, \ell^*, a^*)} N_{(h^*, \ell^*, a^*)}^\tau = \tau$. For any $\cF_H^\tau$-measurable random variable $Z$ taking values in $[0, 1]$, we have
	\begin{align*}
		& \EEs{0}{N_{(h^*, \ell^*, a^*)}^\tau} \frac{16\epsilon^2}{(H-\barH-d)^2} 
		 \geq \EEs{0}{N_{(h^*, \ell^*, a^*)}^\tau} \kl\pa{\frac{1}{2}, \frac{1}{2}+\tepsilon}
		\quad \text{by Lemma \ref{lemma:kl-epsilon}}
		\\
		& = \KL\pa{\P_{0}^{\hist_H^\tau}, \P_{(h^*, \ell^*, a^*)}^{\hist_H^\tau}}
		\quad \text{by Lemma \ref{lemma:kl-between-mdps-with-stopping-time}} \\
		& \geq \kl\pa{\EEs{0}{Z}, \EEs{(h^*, \ell^*, a^*)}{Z}}
		\quad \text{by Lemma \ref{lemma:kl-contraction}}
	\end{align*}
	for any $(h^*, \ell^*, a^*)$, provided that $\tepsilon \leq 1/4$.
	Letting $Z = \indic{\cE_{(h^*,\ell^*,a^*)}^\tau}$ yields
	\begin{align*}
		& \kl\pa{\EEs{0}{Z}, \EEs{(h^*, \ell^*, a^*)}{Z}}
		= \kl\pa{\PPs{0}{\cE_{(h^*,\ell^*,a^*)}^\tau}, \PPs{(h^*, \ell^*, a^*)}{\cE_{(h^*,\ell^*,a^*)}^\tau}}
		\\
		&
		\geq \pa{1 - \PPs{0}{\cE_{(h^*,\ell^*,a^*)}^\tau}} \log\pa{\frac{1}{1-\PPs{(h^*, \ell^*, a^*)}{\cE_{(h^*,\ell^*,a^*)}^\tau}}} -\log(2)
		\quad\text{by Lemma \ref{lemma:kl-delta}}
		\\
		& \geq \pa{1 - \PPs{0}{\cE_{(h^*,\ell^*,a^*)}^\tau}}\log\pa{\frac{1}{\delta}} -\log(2).
	\end{align*}
	Consequently,
	\begin{align*}
		\EEs{0}{N_{(h^*, \ell^*, a^*)}^\tau} \geq \frac{(H-\barH-d)^2}{16\epsilon^2}
		\spa{\pa{1 - \PPs{0}{\cE_{(h^*,\ell^*,a^*)}^\tau}}\log\pa{\frac{1}{\delta}} -\log(2)}.
	\end{align*}
	Summing over all MDP instances, we obtain
	\begin{align}
		& \EEs{0}{\tau}  \geq \sum_{(h^*, \ell^*, a^*)} \EEs{0}{N_{(h^*, \ell^*, a^*)}^\tau} \nonumber \\
		& \geq
		 \frac{(H-\barH-d)^2}{16\epsilon^2}
		\spa{\pa{\barH L A- \sum_{(h^*, \ell^*, a^*)} \PPs{0}{\cE_{(h^*,\ell^*,a^*)}^\tau}}\log\pa{\frac{1}{\delta}} -\barH L A \log(2)} \label{eq:aux-intermediate-bpi-lb}.
	\end{align}
	Now, we have
	\begin{align}
		\sum_{(h^*, \ell^*, a^*)} \PPs{0}{\cE_{(h^*,\ell^*,a^*)}^\tau}
	  	& = \EEs{0}{\sum_{(h^*, \ell^*, a^*)} \indic{\bfPPu{\recpolicy_\tau}{S_{h^*} = s_{\ell^*}, A_{h^*}=a^*} > \frac{1}{2}}} \leq 1 \label{eq:aux-sum-is-bounded-by-one}.
	\end{align}
	Above we used the fact that
		\begin{align*}
			\sum_{(h^*,\ell^*,a^*)} \bfPPu{\recpolicy_\tau}{S_{h^*} = s_{\ell^*}, A_{h^*}=a^*} = \sum_{h^*}\bfPPu{\recpolicy_\tau}{S_{h^*}\in\cL} = 1
		\end{align*}
		since, at a single stage $h^*\in\braces{1+d, \barH+d}$, a leaf state will be reached almost surely. This implies that, if there exists $(h^*,\ell^*, a^*)$ such that $\bfPPu{\recpolicy_\tau}{S_{h^*} = s_{\ell^*}, A_{h^*}=a^*} > \frac{1}{2}$, then, for any other $(h',\ell',a') \neq (h^*,\ell^*, a^*)$, we have   $\bfPPu{\recpolicy_\tau}{S_{h'} = s_{\ell'}, A_{h'}=a'} < \frac{1}{2}$, which proves \eqref{eq:aux-sum-is-bounded-by-one}.
		
	Plugging \eqref{eq:aux-sum-is-bounded-by-one} in \eqref{eq:aux-intermediate-bpi-lb} yields
	\begin{align}
		\EEs{0}{\tau}  & \geq \frac{(H-\barH-d)^2}{16\epsilon^2}
		\spa{\pa{\barH L A- 1}\log\pa{\frac{1}{\delta}} -\barH L A \log(2)} \nonumber\\
		& \geq \barH L A \frac{(H-\barH-d)^2}{32\epsilon^2}
		\log\pa{\frac{1}{\delta}} 	\label{eq:aux-bpi-leaves}
	\end{align}
	where we used the assumption that $\delta \leq 1/16$. The number of leaves $L = (1-1/A)(S-3) + 1/A$ satisfies $L \geq S/4$, since we assume $A \geq 2$, $S \geq 6$. Taking $\barH = H/3$ and with the assumption $d \leq H/3$, we obtain
	\begin{align*}
		\EEs{0}{\tau} \geq \frac{H^3SA}{3456
			 \epsilon^2}\log\pa{\frac{1}{\delta}}.
	\end{align*}
	Finally, the condition $\epsilon \leq H/24$ implies that $\tepsilon\leq 1/4$, as required above.
\end{proof}

	\section{Regret Lower Bound}

Using again change of distributions between MDPs in a class $\cC_{\barH,\varepsilon}$, we prove the following result.

\begin{blockquote}
\begin{theorem}
	\label{theorem:regret-non-stationary}
	Under Assumption \ref{assumption:mdp-structure}, for any algorithm $\histpi$, there exists an MDP $\mdp_{\histpi}$ whose transitions depend on the stage $h$, such that, for $T \geq HSA$
	\[\regret_T(\histpi, \mdp_{\histpi}) \geq \frac{1}{48\sqrt{6}} \sqrt{H^3SAT}\,.\]
\end{theorem}
\end{blockquote}

\begin{proof} Consider the class of MDPs $\cC_{\barH,\varepsilon}$ introduced in Section \ref{sec:hard-mdps}, with $\barH$ and $\varepsilon$ to be chosen later. This class contains a reference MDP $\mdp_0$ and MDPs of the form $\mdp_{(h^*, \ell^*, a^*)}$ parameterized by $$(h^*, \ell^*, a^*) \in \braces{1+d,\ldots, \barH+d}\times \cL \times\cA$$ in which 
	$$\Delta_{(h^*, \ell^*, a^*)}(h, s_i, a) \eqdef \indic{(h, s_i, a)=(h^*, s_{\ell^*}, a^*)} \epsilon.$$ 
As already mentioned, this family of MDPs behave like bandits, hence our proof follows the one for minimax lower bound in bandits (see, e.g., \citealt{bubeck2012regret}).

	\paragraph{Regret of $\histpi$ in $\bm{\mdp_{(h^*, \ell^*, a^*)}}$} 
	The mean reward gathered by $\histpi$ in $\mdp_{(h^*, \ell^*, a^*)}$ is given by
	\begin{align*}
	\EEs{(h^*,\ell^*, a^*)}{\sum_{t=1}^T\sum_{h=1}^H r_h(S_h^t, A_h^t)}
	& = \sum_{t=1}^T \EEs{(h^*,\ell^*, a^*)}{\sum_{h=\barH+d+1}^H \indic{S_h^t = s_g}} \\
	& = (H-\barH-d)\sum_{t=1}^T \PPs{(h^*,\ell^*, a^*)}{S_{\barH+d+1}^t = s_g}.
	\end{align*}
	For any $h \in \braces{1+d, \ldots, \barH+d}$,
	\begin{align}
	& \PPs{(h^*,\ell^*, a^*)}{S_{h+1}^t = s_g} \nonumber \\
	& \!\!=\!
	\PPs{(h^*,\ell^*, a^*)}{S_{h}^t = s_g}\! + \!\frac{1}{2}\PPs{(h^*,\ell^*, a^*)}{S_{h}^t \in \cL} \!+\! \indic{h=h^*}\PPs{(h^*,\ell^*, a^*)}{S_{h}^t = s_{\ell^*}, A_{h}^t=a^*}\epsilon \label{eq:regret-aux-prob}.
	\end{align}
	Indeed, if $S_{h+1}^t = s_g$, we have either $S_h^t = s_g$ or $S_{h+1}^t \in \cL$. In the latter case, the agent has $1/2$ probability of arriving at $s_g$, plus $\epsilon$ if the stage is $h^*$, the leaf is $s_{\ell^*}$ and the action is $a^*$.

	Using the facts that $\PPs{(h^*,\ell^*, a^*)}{S_{1+d}^t = s_g} = 0$ because the agent needs first to traverse the tree and $\sum_{h=1+d}^{\barH+d}\PPs{(h^*,\ell^*, a^*)}{S_{h}^t \in \cL} = 1$ because the agent traverses the tree only once in one episode, we obtain from \eqref{eq:regret-aux-prob} that
	\begin{align*}
	\PPs{(h^*,\ell^*, a^*)}{S_{\barH+d+1}^t = s_g} &= \sum_{h=1+d}^{\barH+d}\frac{1}{2}\PPs{(h^*,\ell^*, a^*)}{S_{h}^t \in \cL} + \indic{h=h^*}\PPs{(h^*,\ell^*, a^*)}{S_{h}^t = s_{\ell^*}, A_{h}^t=a^*}\epsilon\\
	&= \frac{1}{2} + \epsilon \PPs{(h^*,\ell^*, a^*)}{S_{h^*}^t =s_{\ell^*}, A_{h^*}^t=a^*}.
	\end{align*}
	Hence, the optimal value in any of the MDPs is $\rho^* = (H-\barH-d)(1/2+\epsilon)$, which is obtained by the policy that starts to traverse the tree at step $h^*-d$ then chooses to go to the leaf $s_{\ell^*}$ and performs action $a^*$. The regret of $\histpi$ in $\mdp_{(h^*,\ell^*, a^*)}$ is then
	\begin{align*}
	\regret_T\left(\histpi, \mdp_{(h^*,\ell^*, a^*)}\right) =  T(H-\barH-d)\epsilon\pa{1-\frac{1}{T}\EEs{(h^*,\ell^*, a^*)}{N_{(h^*,\ell^*, a^*)}^T}}
	\end{align*}
	where $N_{(h^*,\ell^*, a^*)}^T = \sum_{t=1}^T \indic{S_{h^*}^t = s_{\ell^*}, A_{h^*}^t=a^*}$.

	\paragraph{Maximum regret of $\bm{\histpi}$ over all possible $\bm{\mdp_{(h^*,\ell^*,a^*)}}$} We first lower bound the maximum of the regret by the mean over all instances
	\begin{align}
	\max_{(h^*,\ell^*,a^*)} & \regret_T\left(\histpi, \mdp_{(h^*,\ell^*,a^*)}\right)  \geq \frac{1}{\barH L A}\sum_{(h^*,\ell^*,a^*)} \regret_T\left(\histpi, \mdp_{(h^*,\ell^*,a^*)}\right) \nonumber
	\\
	& \geq T(H-\barH-d)\epsilon\pa{1-\frac{1}{\barH L AT}\sum_{(h^*,\ell^*,a^*)}\EEs{(h^*,\ell^*,a^*)}{N_{(h^*,\ell^*,a^*)}^T}}, \label{eq:regret-thm-main}
	\end{align}
	so that, in order to lower bound the regret, we need an upper bound on the sum of $\EEs{(h^*,\ell^*,a^*)}{N_{(h^*,\ell^*,a^*)}^T}$ over all MDP instances $(h^*,\ell^*,a^*)$. For this purpose, we will relate each expectation to the expectation of the same quantity under the reference MDP $\mdp_0$.

	\paragraph{Upper bound on $\bm{\sum\EEs{(h^*,\ell^*,a^*)}{N_{(h^*,\ell^*,a^*)}^T}}$}  Since $N_{(h^*,\ell^*,a^*)}^T/T \in [0, 1]$, Lemma \ref{lemma:kl-contraction} gives us
	\begin{align*}
	\kl\pa{\frac{1}{T}\EEs{0}{N_{(h^*,\ell^*,a^*)}^T}, \frac{1}{T}\EEs{(h^*,\ell^*,a^*)}{N_{(h^*,\ell^*,a^*)}^T}} \leq \KL\pa{\P_{0}^{\hist_H^T}, \P_{(h^*,\ell^*,a^*)}^{\hist_H^T}}.
	\end{align*}
	By Pinsker's inequality, $(p-q)^2 \leq (1/2)\kl(p, q)$, it implies
	\begin{align*}
	\frac{1}{T}\EEs{(h^*,\ell^*,a^*)}{N_{(h^*,\ell^*,a^*)}^T} \leq \frac{1}{T}\EEs{0}{N_{(h^*,\ell^*,a^*)}^T} +\sqrt{\frac{1}{2}\KL\pa{\P_{0}^{\hist_H^T}, \P_{(h^*,\ell^*,a^*)}^{\hist_H^T}}}
	\end{align*}
	and, by Lemma~\ref{lemma:kl-between-mdps-with-stopping-time}, we know that
	\begin{align*}
	\KL\pa{\P_{0}^{\hist_H^T}, \P_{(h^*,\ell^*,a^*)}^{\hist_H^T}} = \EEs{0}{N_{(h^*,\ell^*,a^*)}^T}\kl(1/2, 1/2+\epsilon)
	\end{align*}
	since $\mdp_0$ and $\mdp_{(h^*,\ell^*,a^*)}$ only differ at stage $h^*$ when $(s, a) = (s_{\ell^*}, a^*)$.
	Assuming that $\epsilon \leq 1/4$, we have $\kl(1/2, 1/2+\epsilon) \leq 4\epsilon^2$ by Lemma \ref{lemma:kl-epsilon}, and, consequently
	\begin{align}
	\label{eq:aux-regrem-thm-1}
	\frac{1}{T}\EEs{(h^*,\ell^*,a^*)}{N_{(h^*,\ell^*,a^*)}^T} \leq \frac{1}{T}\EEs{0}{N_{(h^*,\ell^*,a^*)}^T} +  \sqrt{2}\epsilon \sqrt{\EEs{0}{N_{(h^*,\ell^*,a^*)}^T}}.
	\end{align}
	The sum of $N_{(h^*,\ell^*,a^*)}^T$  over all instances  $(h^*, \ell^*, a^*) \in \braces{1+d,\ldots, \barH+d}\times \cL \times\cA$ is
	\begin{align}
		\label{eq:aux-regret-thm-2}
		\sum_{(h^*,\ell^*,a^*)}N_{(h^*,\ell^*,a^*)}^T = \sum_{t=1}^T \sum_{h^*=1+d}^{\barH+d} \indic{S_{h^*}^t \in \cL} = T
	\end{align}
	since for a single stage $h^* \in \braces{1+d,\ldots, \barH+d}$, we have $S_{h^*}^t \in \cL$ almost surely.

	Summing \eqref{eq:aux-regrem-thm-1} over all instances $(h^*, \ell^*, a^*)$ and using \eqref{eq:aux-regret-thm-2}, we obtain using the Cauchy-Schwartz inequality that
	\begin{align}
		\frac{1}{T}\sum_{(h^*,\ell^*,a^*)}\EEs{(h^*,\ell^*,a^*)}{N_{(h^*,\ell^*,a^*)}^T}
		 &\leq
		1 + \sqrt{2}\epsilon \sum_{(h^*,\ell^*,a^*)}\sqrt{\EEs{0}{N_{(h^*,\ell^*,a^*)}^T}} \nonumber \\
		& \leq 1 + \sqrt{2}\epsilon \sqrt{\barH L A T} \,. \label{eq:aux-regret-thm-3}
	\end{align}

	\paragraph{Optimizing $\bm\epsilon$ and choosing $\bm\barH$} Plugging \eqref{eq:aux-regret-thm-3} in \eqref{eq:regret-thm-main}, we obtain
	\begin{align*}
		\max_{(h^*,\ell^*,a^*)}  \regret_T(\histpi, \mdp_{(h^*,\ell^*,a^*)})  \geq  T(H-\barH-d)\epsilon\pa{1-\frac{1}{\barH L A}-\frac{\sqrt{2}\epsilon \sqrt{\barH L A T} }{\barH L A}}.
	\end{align*}
	The value of $\epsilon$ which maximizes the lower bound is
	$
		\epsilon = \frac{1}{2\sqrt{2}}\pa{1-\frac{1}{\barH L A}}\sqrt{\frac{\barH L A}{T}}
	$
	which yields
	\begin{align}
		\label{eq:aux-regret-leaves}
		\max_{(h^*,\ell^*,a^*)} \regret_T(\histpi, \mdp_{(h^*,\ell^*,a^*)})  \geq
		\frac{1}{4\sqrt{2}}\pa{1-\frac{1}{\barH L A}}(H-\barH-d)\sqrt{\barH L A T}.
	\end{align}

	The number of leaves is $L = (1-1/A)(S-3) + 1/A \geq S/4$, since $A \geq 2$ and $S \geq 6$.
	We choose $\barH = H/3$ and use the assumptions that $A\geq 2$ and $d \leq H/3$ to obtain
	\begin{align*}
		\max_{(h^*,\ell^*,a^*)} \regret_T(\histpi, \mdp_{(h^*,\ell^*,a^*)}) \geq \frac{1}{48\sqrt{6}} H\sqrt{HSAT}.
	\end{align*}
	Finally, the assumption that $\epsilon \leq 1/4$ is satisfied if $T \geq HSA$.
\end{proof}

	\section{Discussion}
\label{sec:discussion}

The lower bounds presented in Theorems \ref{theorem:bpi-non-stationary} and \ref{theorem:regret-non-stationary} hold for MDPs with stage-dependent transitions. As explained in Appendix~\ref{sec:stationary-bounds}, their proof can be easily adapted to the case where the transitions $p_h(\cdot|s,a)$ do not depend on $h$ and the bounds become   $\Omega\pa{\frac{SAH^2}{\epsilon^2}\log\pa{\frac{1}{\delta}}}$ and  $\Omega(\sqrt{H^2SAT})$, respectively.

Our proofs require us to be able to build a full $A$-ary tree containing roughly $S$ nodes whose depth $d$ is small when compared to the horizon $H$, that is $d \leq H/3$ (Assumption \ref{assumption:mdp-structure}). In Appendix~\ref{sec:relax-assumption}, we explain how to obtain the same bounds if we cannot build a full tree, and how the bounds become exponential in $H$ if $d > H/3$.

\acks{The research presented was supported by European CHIST-ERA project DELTA, French Ministry of
	Higher Education and Research, Nord-Pas-de-Calais Regional Council,  French National Research Agency project BOLD (ANR19-CE23-0026-04).}

\bibliography{ref.bib}

\clearpage
\appendix

\newpage


%

\section{Change of Distribution: Proof of Lemma \ref{lemma:kl-between-mdps-with-stopping-time}}
\label{sec:kl-between-mdps-with-stopping-time}

		The pushforward measure of $\P_{\mdp}$ under $I_H^\tau$ is given by
\begin{align*}
\forall T, \; \forall i_H^T \in \histset_H^T, \quad 
\PPsu{\mdp}{I_H^\tau}{i_H^T} = \PPsu{\mdp}{I_H^T}{\tau = T, i_H^T} = \PPs{\mdp}{\tau=T \Big| I_H^T = i_H^T}\PPsu{\mdp}{I_H^T}{i_H^T}.
\end{align*}
If $\PPs{\mdp'}{\tau=T \Big| I_H^T = i_H^T} > 0$ and $\PPsu{\mdp'}{I_H^T}{i_H^T} > 0$, we have 
\begin{align*}
\frac{\PPsu{\mdp}{I_H^\tau}{i_H^T}}{\PPsu{\mdp'}{I_H^\tau}{i_H^T}} 
= \frac{\PPs{\mdp}{\tau=T \Big| I_H^T = i_H^T}\PPsu{\mdp}{I_H^T}{i_H^T}}{\PPs{\mdp'}{\tau=T \Big| I_H^T = i_H^T}\PPsu{\mdp'}{I_H^T}{i_H^T}}
= \frac{\PPsu{\mdp}{I_H^T}{i_H^T}}{\PPsu{\mdp'}{I_H^T}{i_H^T}}
\end{align*}
where we use the fact that $\PPs{\mdp}{\tau=T \Big| I_H^T = i_H^T} = \PPs{\mdp'}{\tau=T \Big| I_H^T = i_H^T}$ since the event $\braces{\tau=T}$ depends only on $I_H^T$. This implies that 
\begin{align*}
\PPsu{\mdp}{I_H^\tau}{i_H^T}\log\pa{\frac{\PPsu{\mdp}{I_H^\tau}{i_H^T}}{\PPsu{\mdp'}{I_H^\tau}{i_H^T}}}
= \PPs{\mdp}{\tau=T \Big| I_H^T = i_H^T}\PPsu{\mdp}{I_H^T}{i_H^T}\log\pa{\frac{\PPsu{\mdp}{I_H^T}{i_H^T}}{\PPsu{\mdp'}{I_H^T}{i_H^T}}}
\end{align*}
under the convention that $0\log(0/0) = 0$. 
Hence,
\begin{align*}
& \KL\pa{\P_{\mdp}^{\hist_H^\tau}, \P_{\mdp'}^{\hist_H^\tau}} 
= \sum_{T=1}^{\infty}\sum_{i_H^T \in \histset_H^T}\PPsu{\mdp}{I_H^\tau}{i_H^T}\log\pa{\frac{\PPsu{\mdp}{I_H^\tau}{i_H^T}}{\PPsu{\mdp'}{I_H^\tau}{i_H^T}}}
\\
& 
= \sum_{T=1}^{\infty}\sum_{i_H^T}\PPs{\mdp}{\tau=T \Big| I_H^T = i_H^T}\PPsu{\mdp}{I_H^T}{i_H^T}\log\pa{\frac{\PPsu{\mdp}{I_H^T}{i_H^T}}{\PPsu{\mdp'}{I_H^T}{i_H^T}}} 
\\
& 
= \sum_{T=1}^{\infty}\sum_{i_H^T}\PPs{\mdp}{\tau=T \Big| I_H^T = i_H^T}
\PPsu{\mdp}{I_H^T}{i_H^T}
\sum_{t=1}^{T}\sum_{h=1}^{H-1}
\log\pa{
	\frac{p_h(s_{h+1}^t|s_h^t, a_h^t)}
	{p_h'(s_{h+1}^t|s_h^t, a_h^t)}
} \\
& 
= \sum_{T=1}^{\infty} \EEs{\mdp}{\indic{\tau=T}
	\sum_{t=1}^{T}\sum_{h=1}^{H-1}
	\log\pa{
		\frac{p_h(S_{h+1}^t|S_h^t, A_h^t)}
		{p_h'(S_{h+1}^t|S_h^t, A_h^t)}}
}
\\
& 
= \EEs{\mdp}{
	\sum_{t=1}^{\tau}\sum_{h=1}^{H-1}
	\log\pa{
		\frac{p_h(S_{h+1}^t|S_h^t, A_h^t)}
		{p_h'(S_{h+1}^t|S_h^t, A_h^t)}}}.
\end{align*}
Now, we apply Lemma~\ref{lemma:condition-in-random-sum} by taking $X_t = \sum_{h=1}^{H-1}
\log\pa{
	\frac{p_h(S_{h+1}^t|S_h^t, A_h^t)}
	{p_h'(S_{h+1}^t|S_h^t, A_h^t)}}$ and $\cF_t = \cF_H^t$. Notice that $X_t$ is bounded almost surely, since when $p_h(S_{h+1}^t|S_h^t, A_h^t) = p_h'(S_{h+1}^t|S_h^t, A_h^t) = 0$, the trajectory containing $(S_h^t, A_h^t, S_{h+1}^t)$ has zero probability. Lemma~\ref{lemma:condition-in-random-sum} and the Markov property give us 
\begin{align*}
& \KL\pa{\P_{\mdp}^{\hist_H^\tau}, \P_{\mdp'}^{\hist_H^\tau}} 
= \EEs{\mdp}{  
	\sum_{t=1}^{\tau}\sum_{h=1}^{H-1}
	\EEs{\mdp}{\log\pa{
			\frac{p_h(S_{h+1}^t|S_h^t, A_h^t)}
			{p_h'(S_{h+1}^t|S_h^t, A_h^t)}} \Big| S_h^t, A_h^t} } \\
& =   \EEs{\mdp}{
	\sum_{t=1}^{\tau}\sum_{h=1}^{H-1}\KL\pa{p_h(\cdot|S_h^t, A_h^t), p_h'(\cdot|S_h^t, A_h^t)} }
= \sum_{s, a, h}
\EEsu{\mdp}{}{ N_{h,s,a}^\tau } \KL\pa{ p_h(\cdot|s,a), p_h'(\cdot|s,a)}.
\end{align*}

\section{PAC-MDP Lower Bound: Proof of Corollary \ref{corollary:pac}}
\label{sec:pac-proof}

 Recall that $\scPAC_\epsilon = \sum_{t=1}^\infty \indic{ \rho^* - \rho^{\pi_t} > \epsilon}$ and let 
\begin{align*}
T(\epsilon, \delta) \eqdef  \frac{1}{6912}\frac{H^3SA}{\epsilon^2}\log\pa{\frac{1}{\delta}} - 1.
\end{align*}

We proceed by contradiction and assume that the claim in Corollary~\ref{corollary:pac} is false. Then we have
\begin{align}
\label{eq:aux-pac-proof}
\text{for all MDP} \;\mdp, \quad
\PPs{\histpi,\mdp}{\scPAC_\epsilon \leq T(\epsilon, \delta)} \geq 1 -\delta.
\end{align}
that is, the algorithm satisfies Definition~\ref{def:pac} with $F_\mathrm{PAC}(S, A, H, 1/\epsilon, \log(1/\delta)) = T(\epsilon, \delta)$. 
In particular, \eqref{eq:aux-pac-proof} holds for any MDP in the class $\cC_{\barH,\tepsilon}$ used to prove Theorem~\ref{theorem:bpi-non-stationary}, for which $\barH = H/3$ and $\tepsilon = 2\varepsilon/(H-\barH - d)$. 

This allows us to build from $\histpi$ a best policy identification algorithm that outputs an $\varepsilon$-optimal policy with probability larger than $1-\delta$ for every MDP in $\cC_{H/3,\tepsilon}$. We proceed as follows: the sampling rule is that of the algorithm $\bm{\pi}$ while the stopping rule is deterministic and set to $\tau \eqdef 2T(\epsilon, \delta) + 1$. Letting $N_t(\pi)$ be the number of times that the algorithm plays a deterministic policy~$\pi$ up to episode $t$, we let the recommendation rule be $\recpolicy_\tau = \argmax_{\pi} N_\tau(\pi)$. 

For every $\cM \in \cC_{H/3,\tepsilon}$, the event $\braces{\scPAC_\epsilon \leq T(\epsilon, \delta)}$ implies $\recpolicy_\tau = \pi^*$. This is trivial for $\mdp_0$, where any policy is optimal, and this holds for any other $\mdp_{(h^*, \ell^* a^*)}\in \cC_{H/3,\tepsilon}$ since there is a unique optimal policy $\pi^*$ and it satisfies $ (\rho^{\pi^*} - \rho^\pi) = 2\epsilon > \epsilon$ in $\mdp_{(h^*, \ell^* a^*)}$ for any other deterministic policy~$\pi$. Hence, if $\recpolicy_\tau \neq \pi^*$, the number of mistakes $\scPAC_\epsilon$ would be larger than $T(\epsilon, \delta)$. 
Thus we proved that the BPI algorithm that we defined satisfies
\begin{align*}
\forall \mdp \in \cC_{H/3,\tepsilon},  \quad 
\PPs{\histpi,\mdp}{\recpolicy_\tau = \pi^*} \geq \PPs{\histpi,\mdp}{\scPAC_\epsilon \leq T(\epsilon, \delta)} \geq 1-\delta.
\end{align*}
Under these conditions, we established in the proof of Theorem \ref{theorem:bpi-non-stationary} that, for $\mdp_0 \in \cC_{H/3,\tepsilon}$, 
\begin{align*}
\tau = \mathbb{E}_{\mdp_0}[\tau] \geq \frac{1}{3456}\frac{H^3SA}{\epsilon^2}\log\pa{\frac{1}{\delta}}
\end{align*}
which yields 
\[2T(\epsilon, \delta) + 1  \geq \frac{1}{3456}\frac{H^3SA}{\epsilon^2}\log\pa{\frac{1}{\delta}}\]
and contradicts the definition of $T(\epsilon,\delta)$.

\section{Recovering the lower bounds for stage-independent transitions}
\label{sec:stationary-bounds}

The proofs of Theorem~\ref{theorem:bpi-non-stationary} and Theorem~\ref{theorem:regret-non-stationary} can be adapted to the case where the transitions $p_h(\cdot|s,a)$ do not depend on $h$. To do so, we need to have a set of hard MDPs with stage-independent transitions. For that, we remove the waiting state $\swait$ and the agent starts at $\sroot$, which roughly corresponds to setting $\barH = 1$ in the proofs, and we take 
$$\Delta_{(h^*, \ell^*, a^*)}(h, s_i, a) \eqdef \indic{(s_i, a)=(s_{\ell^*}, a^*)} \epsilon'$$
to be independent of $h$. We also take $h$-independent rewards as
\begin{align*}
\forall a\in\cA, \quad r_h(s,a) = \indic{s = s_g}.
\end{align*}
Since $\barH = 1$ and no longer $H/3$, the regret bound becomes $\Omega(\sqrt{H^2SAT})$ and the BPI bound becomes $\Omega\pa{\frac{SAH^2}{\epsilon^2}\log\pa{\frac{1}{\delta}}}$.

\section{Relaxing Assumption \ref{assumption:mdp-structure}}
\label{sec:relax-assumption}

In the proofs of Theorems \ref{theorem:bpi-non-stationary} and \ref{theorem:regret-non-stationary}, we use Assumption \ref{assumption:mdp-structure} stating that 
\begin{enumerate}[(i)]
	\item there exists an integer $d$ such that $S = 3 + (A^d-1)/(A-1)$, and
	\item $H \geq 3 d$,
\end{enumerate}
which we discuss below. 

\subsection{Relaxing (i)} 
\label{sec:relax-i}
Assumption (i) makes the proof simpler by allowing us to consider a \emph{full} $A$-ary tree with $S-3$ nodes, which implies that all the leaves are at the same level $d-1$ in the tree. The proof can be generalized to any $S \geq 6$ by arranging the states in a balanced, but not necessarily full, $A$-ary tree. In this case, there might be subset of the leaves at a level $d-1$ and another subset at a level $d-2$, which creates an asymmetry in the leaf nodes. To handle this, we proceed as follows:
\begin{itemize}
	\item First, using $(S-3)/2$ states, we build a balanced $A$-ary tree of depth $d-1$;
	\item For each leaf at depth $d-2$, we add another state (taken among the remaining $(S-3)/2$) as its child.
	\item Any remaining state that was not added to the tree (and is not $\swait$, $s_g$ or $s_b$), can be merged to the absorbing states $s_g$ or $s_b$.
\end{itemize}
This construction ensures that we have a tree with at least $(S-3)/2$ and at most $(S-3)$ nodes, where all the leaves are at the same depth $d-1$, for
\begin{align}
	\label{eq:aux-depth}
	d = \ceil{ \log_A\pa{ (S-3)(A-1)  + 1 }  } \in \spa{\log_A S -1, \log_A S + 2}.
\end{align}
Lemma \ref{lemma:tree} shows that the number of leaves $L$ in this tree satisfies $S \geq L \geq (S-3)/8$. Hence, in the proofs of Theorem \ref{theorem:bpi-non-stationary} (Eq. \ref{eq:aux-bpi-leaves}) and Theorem \ref{theorem:regret-non-stationary} (Eq. \ref{eq:aux-regret-leaves}) , we take $L \geq (S-3)/8$ and obtain lower bounds of the same order.

\subsection{Relaxing (ii)}  
\label{sec:relax-ii}
Equation~\eqref{eq:aux-depth} implies that there exists a constant $c\in[-1, 2]$ such that $d = \log_A S + c$.  Assumption (ii), stating that $H \geq 3 d = 3 \log_A S + 3c$ ensures that the horizon is large enough with respect to the size of the MDP for the agent to be able to traverse the tree down to the rewarding state. If this condition is not satisfied, that is, if $H < 3\log_A S + 3 c$, we have $S \geq A^{\frac{H}{3}- 2}$. In this case, we can build a tree using a subset of the state space containing $\ceil{A^{\frac{H}{3}- 2}}$ states, and merge the remaining $S-\ceil{A^{\frac{H}{3}- 2}}$ states to the absorbing states $s_b$ or $s_g$. In this case, the resulting bounds will replace $S$ by $\ceil{A^{\frac{H}{3}- 2}}$, and become exponential in the horizon $H,$
\begin{align*}
\Omega\pa{\frac{\ceil{A^{\frac{H}{3}- 2}} A H^3}{\epsilon^2}\log\pa{\frac{1}{\delta}}}
\quad \text{and} \quad 
\Omega\pa{\sqrt{H^3 \ceil{A^{\frac{H}{3}- 2}} T}}
\end{align*}
for BPI and regret, respectively.

\subsection{Lower bounds}
The arguments above give us Theorem \ref{theorem:bpi-non-stationary-relaxed}, Corollary \ref{corollary:pac-relaxed} and Theorem \ref{theorem:regret-non-stationary-relaxed} below, which state BPI, PAC-MDP and regret lower bounds, respectively, without requiring Assumption \ref{assumption:mdp-structure}. 

\begin{blockquote}
	\begin{theorem}
		\label{theorem:bpi-non-stationary-relaxed}
		Let $(\histpi, \tau, \recpolicy_\tau)$ be an algorithm that is $(\epsilon, \delta)$-PAC for best policy identification in any finite episodic MDP. Then, if $S \geq 11$, $A \geq 4$ and $H\geq 6$, there exists an MDP $\mdp$ with stage-dependent transitions such that  for $\epsilon \leq H/24$ and $\delta \leq 1/16$,
		\[\EEs{\histpi, \mdp}{\tau} \geq c_1 \min\pa{ S,A^{\frac{H}{3}-2}}
		\frac{H^3 A}{\epsilon^2}
		\log\pa{\frac{1}{\delta}}
		,\]
		where $c_1$ is an absolute constant.
	\end{theorem}
\end{blockquote}
\begin{proof}
	If $S \leq A^{\frac{H}{3}-2}$, then $H \geq 3d$, where $d$ is given in Equation \ref{eq:aux-depth}. In this case, we follow the proof of Theorem \ref{theorem:bpi-non-stationary} up to  Equation \ref{eq:aux-bpi-leaves}, where we take $L \geq (S-3)/8$ according to the arguments in Section \ref{sec:relax-i}. If $S > A^{\frac{H}{3}-2}$, then $H < 3d$ and we follow the arguments in Section \ref{sec:relax-ii}.
\end{proof}
\begin{blockquote}
	\begin{corollary}
		\label{corollary:pac-relaxed}
		Let $\histpi$ be an algorithm that is $(\epsilon,\delta)$-PAC for exploration according to Definition~\ref{def:pac} and that, in each episode $t$, plays a deterministic policy $\pi_t$. Then, under the conditions of Theorem~\ref{theorem:bpi-non-stationary-relaxed}, there exists an MDP $\mdp$ such that 
		\begin{align*}
		\PPs{\histpi,\mdp}{\scPAC_\epsilon > c_2 \min\pa{ S,A^{\frac{H}{3}-2}} \frac{H^3A}{\epsilon^2}\log\pa{\frac{1}{\delta}}-1} > \delta, 
		\end{align*}
		where $c_2$ is an absolute constant.
	\end{corollary}
\end{blockquote}
\begin{proof}
	Analogous to the proof of Corollary \ref{corollary:pac}, using Theorem \ref{theorem:bpi-non-stationary-relaxed} instead of Theorem \ref{theorem:bpi-non-stationary}.
\end{proof}

\begin{blockquote}
	\begin{theorem}
		\label{theorem:regret-non-stationary-relaxed}
		If $S \geq 11$, $A \geq 4$ and $H\geq 6$, for any algorithm $\histpi$, there exists an MDP $\mdp_{\histpi}$ whose transitions depend on the stage $h$, such that, for $T \geq HSA$
		\[\regret_T(\histpi, \mdp_{\histpi}) \geq c_3  \sqrt{ \min\pa{ S,A^{\frac{H}{3}-2}}}\sqrt{H^3AT}\, ,\]
		where $c_3$ is an absolute constant.
	\end{theorem}
\end{blockquote}
\begin{proof}
	If $S \leq A^{\frac{H}{3}-2}$, then $H \geq 3d$, where $d$ is given in Equation \ref{eq:aux-depth}. In this case, we follow the proof of Theorem \ref{theorem:regret-non-stationary} up to  Equation \ref{eq:aux-regret-leaves}, where we take $L \geq (S-3)/8$ according to the arguments in Section \ref{sec:relax-i}. If $S > A^{\frac{H}{3}-2}$, then $H < 3d$ and we follow the arguments in Section \ref{sec:relax-ii}.
\end{proof}

\section{Technical Lemmas}

\begin{lemma}
	\label{lemma:condition-in-random-sum}
	Let $(X_t)_{t\geq 1}$ be a stochastic process adapted to the filtration $(\cF_t)_{t\geq 1}$. Let $\tau$ be a stopping time with respect to $(\cF_t)_{t\geq 1}$ such that $\tau < \infty$ with probability 1. If there exists a constant $c$ such that $\sup_t \abs{X_t} \leq c$ almost surely, then
	\begin{align*}
		\EE{\sum_{t=1}^\tau X_t} = \EE{\sum_{t=1}^\tau \EE{X_t | \cF_{t-1}}}.
	\end{align*}
\end{lemma}
\begin{proof}
	Let $M_n \eqdef \sum_{t=1}^n \pa{X_t - \EE{X_t | \cF_{t-1}}}$. Then, $M_n$ is a martingale and, by Doob's optional stopping theorem, $\EE{M_\tau} = \EE{M_0} = 0$.
\end{proof}


\begin{lemma}
	\label{lemma:kl-epsilon}
	If $\epsilon \in [0, 1/4]$, then $\kl(1/2, 1/2+\epsilon) \leq 4\epsilon^2$. 
\end{lemma}
\begin{proof}
	Using the inequality $-\log(1-x) \leq 1/(1-x) - 1$ for any $0<x<1$, we obtain 
	\begin{align*}
	\kl(1/2, 1/2+\epsilon) = -\frac{1}{2} \log(1-4\epsilon^2) \leq \frac{1}{2}\pa{\frac{1}{1-4\epsilon^2} - 1} = \frac{2\epsilon^2}{1-4\epsilon^2}.
	\end{align*}
	If $\epsilon \leq 1/4$, then $1-4\epsilon^2 \geq 3/4 > 1/2$, which implies the result.
\end{proof}

\begin{lemma}
	\label{lemma:kl-delta}
	For any $p, q \in [0, 1]$,
	\begin{align*}
		\kl(p, q) \geq (1-p)\log\pa{\frac{1}{1-q}} - \log(2).
	\end{align*}
\end{lemma}
\begin{proof}
	If follows from the definition of $\kl(p, q)$ and the fact that the entropy $H(p) \eqdef p\log(1/p) + (1-p)\log(1/(1-p))$ satisfies $H(p) \leq \log(2)$:
	\begin{align*}
		\kl(p, q) & = p \log\pa{\frac{p}{q}} + (1-p)\log\pa{\frac{1-p}{1-q}} \\
		& = (1-p)\log\pa{\frac{1}{1-q}} + (1-p)\log\pa{\frac{1}{1-q}} + p\log\pa{\frac{1}{q}} - H(p) \\
		& \geq (1-p)\log\pa{\frac{1}{1-q}} -\log(2).
	\end{align*}
\end{proof}


\begin{lemma}
	\label{lemma:tree} Let $L$ be the number of leaves in a balanced $A$-ary tree with $S$ nodes and $A\geq 2$. Then,  $L \geq S/4$.
\end{lemma}
\begin{proof}
	Let $d$ be the depth of the tree. There exists an integer $R$ such that $0 < R \leq A^d$ such that 
	\begin{align*}
		S = \frac{A^{d} - 1}{A-1} + R.
	\end{align*}
	The number of leaves is given by $L = R + A^{d-1} - \ceil{\frac{R}{A}}$.
	We consider two cases: either $ \frac{A^{d} - 1}{A-1} \leq \frac{S}{2}$ or $ \frac{A^{d} - 1}{A-1} > \frac{S}{2}$. If $\frac{A^{d} - 1}{A-1} \leq \frac{S}{2}$, we have $R \geq S/2$ which implies $L \geq S/2 > S/4$. If $ \frac{A^{d} - 1}{A-1} > \frac{S}{2}$, we have $L \geq A^{d-1} > \frac{1}{A} + \frac{S}{2}\pa{1-\frac{1}{A}} \geq S/4$.
\end{proof}

\end{document}